\documentclass[final]{l4dc2026}

\usepackage{mathtools}
\usepackage{derivative}
\usepackage{xcolor}
\usepackage{bm}
\usepackage[font=small, labelfont=bf]{caption}
\usepackage{enumitem}

\usepackage[export]{adjustbox} 
\usepackage{tikzducks}
\usepackage{wrapfig}
\usepackage{multirow}  

\usepackage{xurl}
\usepackage{hyperref} 
\hypersetup{
    colorlinks=true,
    linkcolor=black,
    urlcolor=magenta,
}

\newtheorem*{lemma*}{Lemma} 
\newtheorem*{theorem*}{Theorem} 

\newcommand{\R}{\mathbb{R}}

\newcommand{\mc}[1]{\mathcal{#1}}
\newcommand{\T}{^\top}

\newcommand{\rank}{\operatorname{rank}}

\newcommand{\norm}[1]{\left\lVert#1\right\rVert}
\DeclareMathOperator{\diag}{diag}

\newcommand{\bzero}{\mathbf{0}}

\newcommand{\bSigma}{\mathbf{\Sigma}}
\newcommand{\bff}{\mathbf{f}}
\newcommand{\bg}{\mathbf{g}}

\newcommand{\bk}{\mathbf{k}}

\newcommand{\bo}{\mathbf{o}}

\newcommand{\bq}{\mathbf{q}}

\newcommand{\bu}{\mathbf{u}}

\newcommand{\bx}{\mathbf{x}}

\newcommand{\bz}{\mathbf{z}}

\newcommand{\bA}{\mathbf{A}}
\newcommand{\bB}{\mathbf{B}}

\newcommand{\bD}{\mathbf{D}}
\newcommand{\bE}{\mathbf{E}}

\newcommand{\bH}{\mathbf{H}}
\newcommand{\bI}{\mathbf{I}}
\newcommand{\bJ}{\mathbf{J}}

\newcommand{\bM}{\mathbf{M}}

\newcommand{\bP}{\mathbf{P}}
\newcommand{\bQ}{\mathbf{Q}}

\newcommand{\bU}{\mathbf{U}}
\newcommand{\bV}{\mathbf{V}}

\newcommand{\bpi}{\bm{\pi}}
\newcommand{\bpsi}{\bm{\psi}}

\newcommand{\bphi}{\bm{\phi}}

\newcommand{\bOmega}{\bm{\Omega}}

\newcommand{\M}{\mathcal{M}}

\newcommand{\flow}{\bm{\varphi}}

\title[HALO]{HALO: Hybrid Auto-encoded Locomotion with Learned Latent Dynamics, Poincar\'e Maps, and Regions of Attraction}

\usepackage{times}

\author{%
 \Name{Blake Werner*} \Email{bwerner@caltech.edu}\\
 \Name{Sergio Esteban*} \Email{sesteban@caltech.edu}\\
 \Name{Massimiliano de\;Sa} \Email{mdesa@caltech.edu}\\
 \addr Caltech, Pasadena, CA, USA
 \AND
 \Name{Max H. Cohen} \Email{mhcohen2@ncsu.edu}\\
 \addr NC State University, Raleigh, NC, USA
 \AND
 \Name{Aaron D. Ames} \Email{ames@caltech.edu}\\
 \addr Caltech, Pasadena, CA, USA
}

\begin{document}

\maketitle

\begin{abstract}%
Reduced-order models are powerful for analyzing and controlling high-dimensional dynamical systems. Yet constructing these models for complex hybrid systems such as legged robots remains challenging. Classical approaches rely on hand-designed template models (e.g., LIP, SLIP), which, though insightful, only approximate the underlying dynamics. In contrast, data-driven methods can extract more accurate low-dimensional representations, but it remains unclear when stability and safety properties observed in the latent space meaningfully transfer back to the full-order system.
To bridge this gap, we introduce HALO (Hybrid Auto-encoded Locomotion), a framework for learning latent reduced-order models of periodic hybrid dynamics directly from trajectory data. HALO employs an autoencoder to identify a low-dimensional latent state together with a learned latent Poincar\'e map that captures step-to-step locomotion dynamics. This enables Lyapunov analysis and the construction of an associated region of attraction in the latent space, both of which can be lifted back to the full-order state space through the decoder. Experiments on a simulated hopping robot and full-body humanoid locomotion demonstrate that HALO yields low-dimensional models that retain meaningful stability structure and predict full-order region-of-attraction boundaries.
\end{abstract}


\section{Introduction}
Reduced-order model (ROM) techniques aim to approximate high-dimensional dynamical systems with lower-dimensional representations. Such reductions enable more efficient verification, control synthesis, and planning for systems whose full-order model (FOM) dynamics may be too complex to analyze directly. Classical approaches to model reduction in control include balanced truncation and related projection-based methods \cite{antoulas2005approximation,wilcox2015survey}, moment-matching for linear and nonlinear systems \cite{astolfi2010moment}, and methods based on invariant manifolds \cite{haller2016ssm}. These approaches share conceptual ties with notions of abstraction in hybrid and symbolic control \cite{tabuada2009hybrid,girard2007simulation}, where simpler surrogate systems are constructed to preserve relevant dynamical properties of the original one.

While these methods provide rigorous guarantees for structured or weakly nonlinear systems, extending them to high-dimensional robotic systems, particularly those with contact-rich dynamics, remains challenging. Motivated by these limitations, recent works
learn ROMs from data.
Approaches based on dynamic mode decomposition, PCA, and proper orthogonal decomposition \cite{brunton2022book} identify dominant low-dimensional structure in trajectories, while autoencoders learn nonlinear embeddings that map between full-order states and latent coordinates. A central challenge, however, is determining when properties such as stability, safety, or invariance in the latent space faithfully transfer back to the original full-order dynamics. Recent work has begun to establish sufficient conditions and approximation guarantees in this direction \cite{lutkus2025latent,bajcsy2025latent}, though practical instantiations for hybrid locomotion remain limited.

ROMs play a particularly influential role in legged robotics, where simple models such as the linear inverted pendulum (LIP) \cite{kajita20013d} and the spring-loaded inverted pendulum (SLIP) \cite{blickhan1989spring} capture the salient dynamics of center-of-mass motion \cite{raibert1986legged,wensing2017templates}. These templates are effective for control design and gait stabilization, but are inherently approximate: they do not directly arise from a systematic reduction of the FOM, and their accuracy can degrade outside of nominal operating regimes. Thus, recent work has learned reduced-order locomotion models from data \cite{Castillo2024DataDriven,Chen2024Reinforcement,Chen2024Beyond,li2024fld,Starke2022DeepPhase}, though questions remain regarding how to analyze their stability properties. 

\paragraph{Overview and Contributions:}
In this paper, we propose HALO (Hybrid Autoencoded LOcomotion), a framework for constructing and analyzing ROMs of hybrid robotic locomotion using autoencoder-based latent representations. Our approach represents the high-dimensional hybrid dynamics as a discrete-time system via a Poincar\'e return map \cite{westervelt2018feedback}. We then learn an autoencoder that embeds this return map into a low-dimensional latent space, yielding a latent Poincar\'e map that captures the step-to-step evolution of the locomotion gait cycle.

We first motivate this reduction in the linear setting, then extend to the nonlinear case by leveraging the existence of low-dimensional invariant surfaces associated with periodic orbits. This perspective provides practical guidance for selecting the latent dimension and clarifies when the autoencoder can represent the orbit geometry with negligible reconstruction error. Once the latent dynamics are identified, we perform Lyapunov stability analysis directly in the latent space and subsequently lift the corresponding region of attraction to the full-order state space via the decoder.

We evaluate HALO 
across a suite of hybrid systems.
Across all systems, we show that stability properties inferred from the latent Poincar\'e dynamics predict stability in the full-order system, even when it is controlled via reinforcement learning policies.
\section{Preliminaries}
%
\paragraph{Notation:} The Euclidean and Frobenius norms are denoted by $\norm{\cdot}$ and $\norm{\cdot}_F$, respectively. The restriction of a function $f: A \to B$ to a subset $U \subseteq A$ is denoted $f|_{U}: U \to B$. The inclusion map of a subset $U \subseteq A$ into $A$, which takes each $u \in U$ to $u \in A$, is denoted $\iota_U: U \hookrightarrow A$.

\paragraph{Model Reduction \& Autoencoders:} 
High-dimensional dynamical systems often have an underlying low-dimensional structure. We formalize extracting low-dimensional structure from dynamical systems as follows. Consider a discrete-time dynamical system: 
\begin{align}
    \bx_{k+1} = \bff(\bx_k),\quad \bx_k \in \R^{n_x},\quad\bff: \R^{n_x} \to \R^{n_x}.
\end{align}
We refer to $\bx_k \in \R^{n_x}$ as the \textit{full-order state}, and $\bff$ as the FOM. To obtain a ROM, we must \textit{compress} $\bx$ to a lower-dimensional space, and $\bff$ to a lower-dimensional model. Additionally, we must have a way to \textit{decompress} from this reduced state back to $\bx$, motivating the use of \textit{autoencoders}.

The process of compression and decompression is performed by \textit{encoder} and \textit{decoder} maps. First, we define $n_z < n_x$ to be the dimension of our ROM. An encoder  $\bE: \R^{n_x} \to \R^{n_z}$ performs compression, and a decoder $\bD: \R^{n_z} \to \R^{n_x}$ performs decompression. Since $\R^{n_z}$ represents the hidden low-dimensional structure, we refer to $\R^{n_z}$ as the \textit{latent space}, and $n_z$ as the \textit{latent dimension}. We say an encoder-decoder pair $(\bE, \bD)$ is \textit{perfect} if $\bD \circ \bE(\bx) = \bx$ for all $\bx \in \R^{n_x}$, i.e, no information is lost from compression to decompression. In this spirit, to find a low-dimensional representation of $\bff$, we seek a map $\bg: \R^{n_z} \to \R^{n_z}$ that captures the important features of $\bff$. We refer to the map $\bg: \R^{n_z} \to \R^{n_z}$ as the \textit{reduced-order model}. This model determines a recurrence:
\begin{align}
    \bz_{k+1} &= \bg(\bz_k), \; \bz_k \in \R^{n_z},
\end{align}
in the latent space. In a perfect world, one would seek a perfect encoder-decoder pair $(\bE, \bD)$ for which $\bff(\bx) = \bD \circ \bg \circ \bE(\bx)$, for all $\bx \in \R^{n_x}$. Such an encoder-dynamics pair would perfectly compress the original high-order system down to a low-dimensional latent space. 

For complex, real-world systems, perfect compression in this sense is unrealistic. Likewise, one cannot analytically determine optimal encoders, decoders, and ROMs. To find a ``good'' encoder-decoder and ROM, one practically represents each by a neural network $\bE_{\boldsymbol{\phi}}$, $\bD_{\boldsymbol{\psi}}$, with parameters $\bphi$ and $\bpsi$, respectively, and encodes the goals above in appropriately designed \textit{loss functions}. 

\paragraph{From Discrete to Continuous with Hybrid Systems:} 
Above, we discussed the construction of autoencoder models for discrete-time dynamical systems. In robotics, however, there are a wide class of systems that do not fall naturally into this category. Robotic systems, such as legged robots, are
\textit{hybrid systems} \cite{westervelt2018feedback} that exhibit both discrete \textit{and} continuous-time behavior:
\begin{align}
\mathcal{H}=
    \begin{cases}
        \dot \bx = \mathbf{f}_c(\bx), & \bx \in \R^{n_x} \setminus \mathcal S\\
        \bx^+ = \mathbf{\Delta}(\bx^-), & \bx^- \in \mathcal S,
    \end{cases}
\end{align}
where $\mathcal S \subseteq \R^{n_x}$ is a set that triggers a discrete change in behavior. For legged robots, this change is triggered by a foot strike with the ground. Under certain conditions, one may construct a map $\bff:\mathcal{S}\rightarrow\mathcal{S}$, termed the \textit{Poincar\'e map} that maps a state $\bx \in \mathcal S$ to the state $\bff(\bx)\in\mathcal{S}$ at which the system next intersects $\mathcal S$; for a legged robot, $\bff$ maps from the state at one foot strike to the state at the next foot strike. Thus, the Poincar\'e map induces a discrete-time system $\bx_{k+1} = \bff(\bx_k)$ capturing the 
step-to-step dynamics.
Studying the stability of the Poincar\'e map affords insights into the behavior of the hybrid system. Details on hybrid systems and Poincar\'e maps are provided in Appendix~\ref{sec:fom_hybrid_dynamics}.
\begin{figure}
    \centering
    \includegraphics[width=0.8\linewidth]{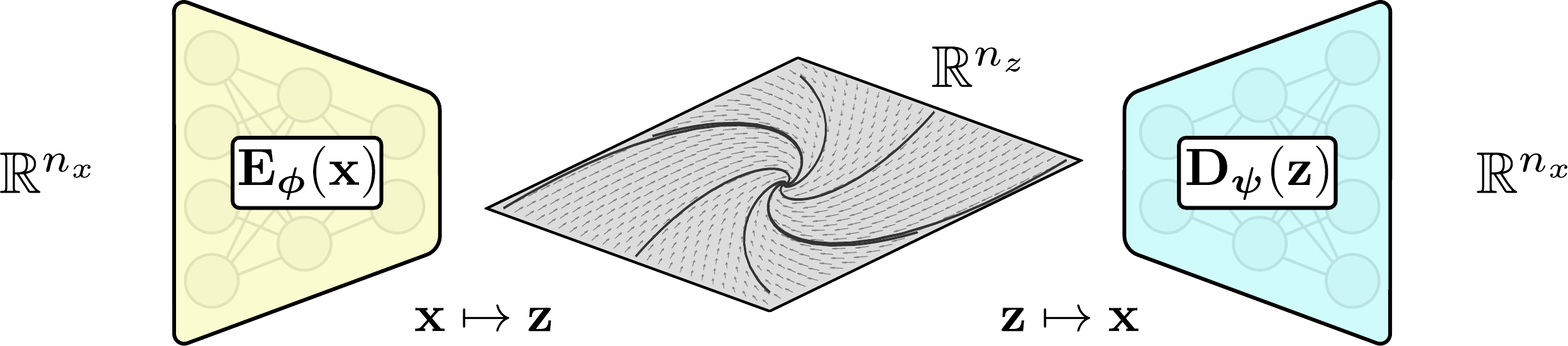}
    \vspace{-5pt}
    \caption{Autoencoders enable learning of a reduced-order dynamics model in a latent space.}
    \label{fig:architecture}
    \vspace{-10pt}
\end{figure}


\section{Insights from Linear and Nonlinear Systems}\label{sec:insights}
In the previous section, we discussed autoencoders, ROMs, and Poincar\'e maps. Autoencoders provide a means to compute ROMs from complex, high-dimensional data; ROMs enable analyzing complex systems using simple models; Poincar\'e maps yield insight into the stability of hybrid systems from simpler discrete-time models. 
We now bring these threads together to study the behavior of hybrid systems. We use autoencoders to learn ROMs of Poincar\'e maps and analyze their stability in the latent space to infer stability properties of the full hybrid system.
We begin by studying simple yet insightful motivation from linear and nonlinear dynamics, which provide templates for our learning-based approach. Then, we outline an autoencoder method for learning ROMs of Poincar\'e maps. 
We defer all proofs in this section to Appendix \ref{appendix:proofs}.

\paragraph{Motivation from Linear Systems:}
To motivate the nonlinear problem, we first perform a simple study for linear systems. Consider a full-order model governed by discrete-time LTI dynamics $\bx_{k+1} = \bA \bx_k$, where $\bx_k \in \R^{n_x}$. Our goal is to reduce the system to a lower-dimensional linear space. To this end, we fix a \textit{reduced dimension} $n_z$, smaller than $\rank(\bA)$. To perform dimensionality reduction, we define three matrices: the \textit{encoder matrix} $\bE \in \R^{n_z \times n_x}$, the \textit{decoder matrix} $\bD \in \R^{n_x \times n_z}$, and the \textit{latent dynamics matrix} $\bQ \in \R^{n_z \times n_z}$. How do we pick $\bD, \bQ, \bE$ such that trajectories of the autoencoder model, $\hat \bx_k = \bD \bQ^k \bE \bx_0$, resemble those of the full-order model? First, we consider the problem of matching \textit{one step} of the dynamics. Using the Eckart-Young theorem \cite{calafiore2014optimization}, we solve for the $\bE, \bD, \bQ$ minimizing the worst-case one-step error.
\begin{lemma}[One-Step Optimal Matching]\label{lem:one-step}
    Let $\bA \in \R^{n_{x}\times n_x}$ and $n_z \in (0, \rank(\bA)]_{\mathbb Z}$. Consider:
    \begin{align}\label{eq:lemma1}
        \inf_{\bE \in \R^{n_z \times n_x}, \bD \in \R^{n_x \times n_z}, \bQ \in \R^{n_z \times n_z}} \sup_{\bx_0 \in \R^{n_x}, \norm{\bx_0} = 1}\norm{\bA \bx_0 - \bD \bQ \bE \bx_0}^2.
    \end{align}
     If $\bA = \bU \bSigma \bV^\top$ is an SVD of $\bA$, then $\bD = \bU_{n_z}$, $\bQ = \bSigma_{n_z}$, $\bE = \bV_{n_z}^\top$, where $\bU_{n_z}, \bV_{n_z}$ are the first $n_z$ columns of $\bU, \bV$ and $\bSigma_{n_z} = \diag(\sigma_1, ..., \sigma_{n_z})$, are optimal solutions to \eqref{eq:lemma1}.
\end{lemma}
Now, we consider the multi-step version of Lemma \ref{lem:one-step}. We seek a problem with a simple, closed-form solution that provides insight into the more general autoencoder problem. Inspired by the $\mathcal H_2$ model reduction problem \cite{dullerud2013course}, we consider the optimization problem:
\begin{align}\label{eq:model-red-prob}
    \inf_{\bE \in \R^{n_z \times n_x}, \bD \in \R^{n_x \times n_z}, \bQ \in \R^{n_z \times n_z}} \sum_{k = 0}^\infty \norm{\bA^k - \bD \bQ^k \bE}_F^2.
\end{align}
Though this problem is not generally solvable in closed-form, there are certain special cases for which closed-form solutions exist and are informative.
\begin{lemma}[Diagonal Model Reduction]\label{lem:diag-model-red}
    Suppose $\bA = \diag(\lambda_1, ..., \lambda_{n_x})$, with $|\lambda_1| \geq ... \geq |\lambda_{n_x}|$. Let $n_z \in (0, \rank(\bA)]_{\mathbb Z}$. If the unstable and center subspaces of $\bA$ together have dimension $\leq n_z$, then \eqref{eq:model-red-prob} is feasible and an optimal solution is given by the $n_z$-truncation:
    \begin{align}
        \hat \bD = \begin{bmatrix}
            I_{n_z}\\
            0_{(n_x-n_z) \times n_z}
        \end{bmatrix}, \quad \hat \bQ = \diag(\lambda_1, ..., \lambda_{n_z}), \quad \hat \bE = \begin{bmatrix}
            I_{n_z} & 0_{n_z \times (n_x-n_z)}
        \end{bmatrix}.
    \end{align}
\end{lemma}
Lemma \ref{lem:diag-model-red} suggests a heuristic for viewing the nonlinear dimensionality reduction problem: if a latent dimension is capable of capturing complex, unstable system behavior, and all remaining modes are stable, then the latent dimension is capable of producing a ``good" reduced-order model. In what follows, we further develop this perspective in the nonlinear setting.

\paragraph{Motivation from Nonlinear Systems:}
To begin our nonlinear analysis, recall the \textit{manifold hypothesis}, which posits that data points in real-world, high-dimensional data sets are often clustered around low-dimensional \textit{submanifolds} \cite{fefferman2016testing}. Now, we apply this hypothesis, which underlies various dimensionality reduction methods, to our model reduction problem.

\begin{figure}
    \centering
    \includegraphics[width=0.8\linewidth]{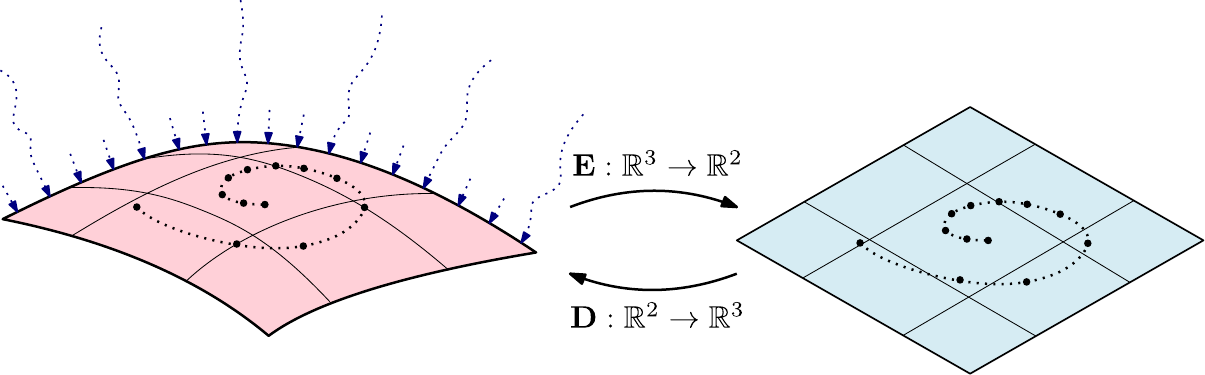}
    \vspace{-5pt}
    \caption{An \textit{attractive invariant manifold} for a discrete-time system (in red) can have a lower dimensional embedding (in blue). This embedding informs a choice of latent dimension and reduced-order model.}
    \label{fig:attr-inv-manif}
    \vspace{-15pt}
\end{figure}

Suppose we collect trajectory data $\mathfrak{D} = \{\{\bx_k\}_{k = 0}^K: \bx_0 \in \mathcal X \subseteq \R^n, \; \bx_{k+1} = \bff(\bx_k)\}$ from a variety of initial conditions $\bx_0 \in \mathcal X \subseteq \R^n$. By the manifold hypothesis, suppose the trajectory data concentrates around a submanifold of $\R^n$. Since we sample $\bx_0$ from across the state space, trajectories clustering around a submanifold suggest the existence of an \textit{attractive invariant manifold} to which nearby trajectories converge. To gain insight into the nonlinear autoencoder problem, we thus present a simple study relating invariant manifolds and autoencoder reduced-order models.

First, employing the classic \textit{Whitney embedding theorem} \cite[Thm. 6.19]{lee2013smooth}, we establish a basic relationship between the dimensionality of an invariant manifold $\M$ (for which $\bff(\M) = \M$) and the dimension of latent space required to fully represent $\M$ and the dynamics on it.

\begin{lemma}[Model Reduction \& Invariant Manifolds]\label{lem:inv-manif}
    Let $\bff: \R^{n_x} \to \R^{n_x}$ be a map defining a discrete-time system with smooth invariant manifold $\M \subseteq \R^{n_x}$. For $\dim(\M) = k$:
    \begin{enumerate}
        \item \underline{Perfect encoder-decoder pair on $\M$}: there exists a smooth encoder $\bE: U \subseteq \R^{n_x} \to \R^{n_z}$ and a smooth decoder $\bD: W \subseteq \R^{n_z} \to \R^{n_x}$ satisfying $\bD(\bE(\bx)) = \bx$  for all $\bx \in \M$, where $n_z \leq 2k$ and $U, W$ are open neighborhoods of $\M$ and $\bE(\M)$ respectively, 
        \item \underline{Latent system}: there exists a map $\bg: \R^{n_z} \to \R^{n_z}$ satisfying $\bff(\bx) = \bD \circ \bg \circ \bE(\bx) \; \forall \bx \in \M$.
    \end{enumerate}
\end{lemma}
\begin{remark}
    Further reductions in $n_z$ are possible assuming compactness of $\M$ and analogues of (1) and (2) on open subsets of $\M$ instead of on its entirety \cite{kvalheim2025autoencoding}.
\end{remark}
Using $\bE $ and $ \bD$ constructed in the proof of Lemma \ref{lem:inv-manif}, we make the simple observation that stability of the latent dynamics on $\bE(\M)$ implies stability of the original system on $\M$.
\begin{theorem}[Transfer of Stability on $\M$]\label{thm:transf-stb}
    Consider the setting of Lemma \ref{lem:inv-manif} and its proof. Let $\tilde \bg : \bE(\M) \to \bE(\M)$ be the restriction of $\bg$ to $\bE(\M)$ and $\tilde \bff: \M \to \M$ the restriction of $\bff$ to $\M$. Suppose $\bz^* \in \bE(\M)$ is a locally asymptotically stable fixed point of $\tilde \bg$. Then, $\bx^* = \bD(\bz^*)$ is a locally asymptotically stable fixed point of $\tilde \bff$. Further, if $\bz$ belongs to the region of attraction of $\bz^*$ for $\tilde \bg$, then $\bD(\bz)$ belongs to the region of attraction of $\bx^*$ for $\tilde \bff$.
\end{theorem}
\begin{remark}
    Here, stability is with respect to open subsets of $\bE(\M)$ and $\M$, not $\R^{n_z}$ and $\R^{n_x}$.
\end{remark}
Thus, stability in the latent representation of the invariant manifold translates to stability in the full-order representation of the invariant manifold. Practically---as we demonstrate below---stability on the invariant manifold can be concluded by exhibiting an asymptotic Lyapunov function for $\bg: \R^{n_z} \to \R^{n_z}$; such a Lyapunov function automatically restricts to an asymptotic Lyapunov function on $\bE(\M)$, which provides the conditions for Theorem \ref{thm:transf-stb}.

To make a stronger conclusion regarding stability for the full-order dynamics $\bff: \R^{n_x} \to \R^{n_x}$ on a neighborhood of $\bx^*$ in $\R^{n_x}$, rather than stability on a neighborhood in $\M$, one may appeal to the literature on attractive invariant manifolds \cite{wiggins1994normally}. Informally, stability on the invariant manifold plus sufficiently fast convergence normal to the invariant manifold implies stability of the full system. Within the scope of this work, such convergence remains an assumption.

\section{Methods}\label{sec:methods}
\begin{figure}[t]
    \centering
    \includegraphics[width=0.9\linewidth]{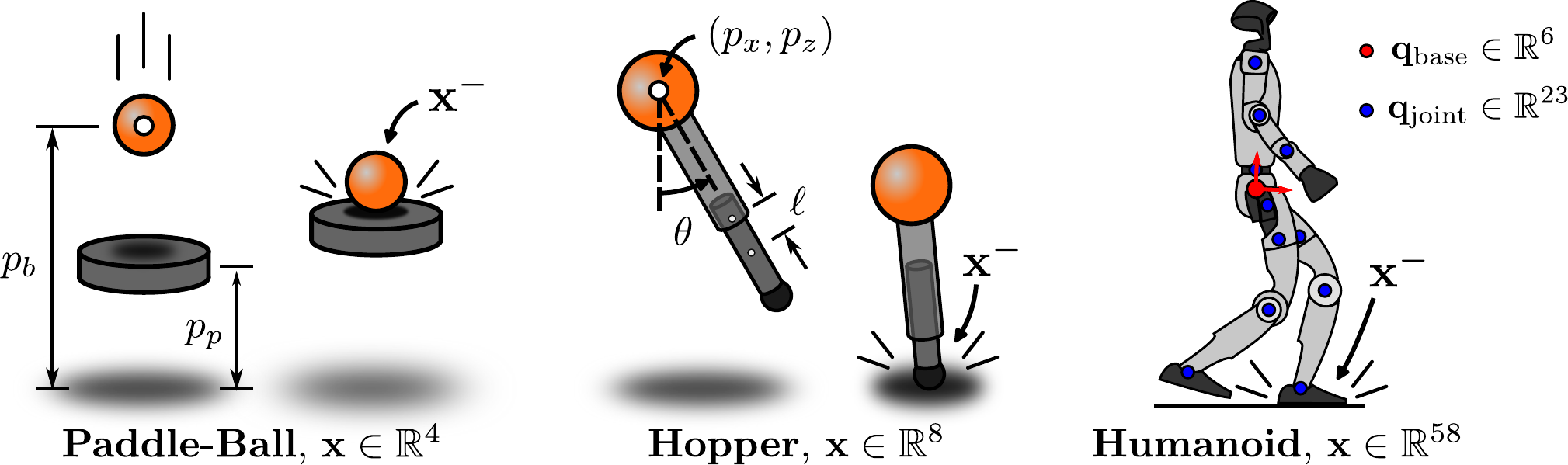}
    \caption{The three systems of study, their configurations, and their impacts.}
    \label{fig:example_systems}
    \vspace{-10pt}
\end{figure}
We recall that the stability of a hybrid system can be characterized by its impact-to-impact dynamics. We consider a closed-loop hybrid system with state $\bx$ and feedback controller $\mathbf{k}(\bx)$, learned through reinforcement learning (RL). To study stability of this system, we extract sequences of pre-impact states,
\begin{equation}
    \mathfrak{D} = \{\{\bx_k\}_{k = 0}^K: \bx_0 \in \mathcal X_0 \subset \R^{n_x}, \; \bx_{k+1} = \bff(\bx_k)\}    
\end{equation}
where $\mathbf{f}: \R^{n_x} \rightarrow \R^{n_x} $ is the Poincar\'e map of the closed-loop system and each $\bx_k \coloneq \bx^-$ is its state just before an impact. The configurations and pre-impact states of the systems considered here---a paddle-ball, a planar hopper, and a Unitree G1 humanoid robot---are illustrated in Figure~\ref{fig:example_systems}. Using the dataset $\mathfrak D$, we train an autoencoder-based ROM that captures the Poincar\'e dynamics in a low-dimensional latent space. We then use the learned latent dynamics to predict trajectories and estimate regions of attraction of the FOM. The remainder of this section details each component of this pipeline\footnote{Full code is available at: \href{https://github.com/sesteban951/halo-latent-locomotion}{\texttt{github.com/sesteban951/halo-latent-locomotion}}.}.
\paragraph{Reinforcement Learning Controllers:}
As described above, the feedback controller $\bk(\bx)$ is an RL policy $\bpi_{\xi}(\bo)$, where the observation \(\bo\) is a function of the state for the paddle-ball and hopper, and a function of the state together with auxiliary variables for the humanoid. The paddle-ball policy is trained to regulate the ball at a desired height, whereas the hopper and humanoid policies are trained to maintain a constant forward velocity.

We train controllers for the paddle-ball and hopper using the actor-critic PPO algorithm implemented in Brax \cite{brax2021github}. The learned policy directly outputs bounded torque commands at 50 Hz. We train a locomotion controller for the G1 humanoid using the actor-critic PPO algorithm implemented in mjlab \cite{zakka2026mjlablightweightframeworkgpuaccelerated} with the \texttt{rsl\_rl} library. The learned policy outputs joint position setpoints at 50 Hz, which are tracked by joint-level PD controllers to produce motor torques. Additional implementation details are provided in Appendix~\ref{sec:rl_policy}.


%
\paragraph{Poincar\'e Data Collection:}
To collect training data, we simulate the closed-loop system in MuJoCo MJX~\cite{todorov2012mujoco} for a fixed time horizon, using the JAX backend for the paddle-ball and hopper systems and the Warp backend for the G1. At each impact, we use MuJoCo sensor measurements to record the generalized position and velocity in a contact-relative frame---in the world frame for the paddle-ball, relative to the foot position for the hopper, and in a yaw-aligned foot frame for the humanoid. After pruning faulty Poincar\'e data, such as impacts with sliding contact or contact chatter, we obtain clean sequences of Poincar\'e returns.

\begin{remark}
In practice, impacts can be estimated without dedicated contact sensors. Kinematic methods use foot height and velocity; dynamic methods use changes in generalized velocities, accelerations, or momentum;  force methods use joint torques and ground reaction force estimates.
\end{remark}
%
%
\paragraph{Autoencoder Architecture:}
We parameterize the autoencoder using three networks: an \textit{encoder} that projects the FOM state onto a lower-dimensional latent state, $\bz_k = \bE_{\bphi}(\bx_k)$, a \textit{decoder} that reconstructs the FOM state from the latent state, $\hat{\bx}_k = \bD_{\bpsi}(\bz_k)$, and a \textit{latent dynamics} network that propagates the latent state forward via learned residual ROM dynamics,
\begin{equation}
    \bz_{k+1} = \bg_{\boldsymbol\rho}(\bz_k) \coloneq \bz_k + \bar{\bg}_{\boldsymbol\rho}(\bz_k).
\end{equation}
The encoder and decoder are mirrored MLPs with swish activations, sized $[64, 32, 16]$ and $[16, 32, 64]$ for the hopper and paddle-ball systems, and $[256, 128, 64]$ and $[64, 128, 256]$ for the humanoid. The latent dynamics network is an MLP with swish activations, sized $[64, 64, 64]$ for the hopper and paddle-ball systems, and $[128, 128, 128]$ for the humanoid. The latent dimension $n_z$ is 2 for the paddle-ball, 4 for the hopper, and 12 for the humanoid. More details are provided in Section \ref{sec:autoencoder_hyperparams}.
%
\paragraph{Loss Function:} \label{sec:learning_latent_dynamics}
\begin{wrapfigure}{r}{0.45\linewidth}
    \vspace{-1pt} 
    \centering
    \includegraphics[width=0.8\linewidth]{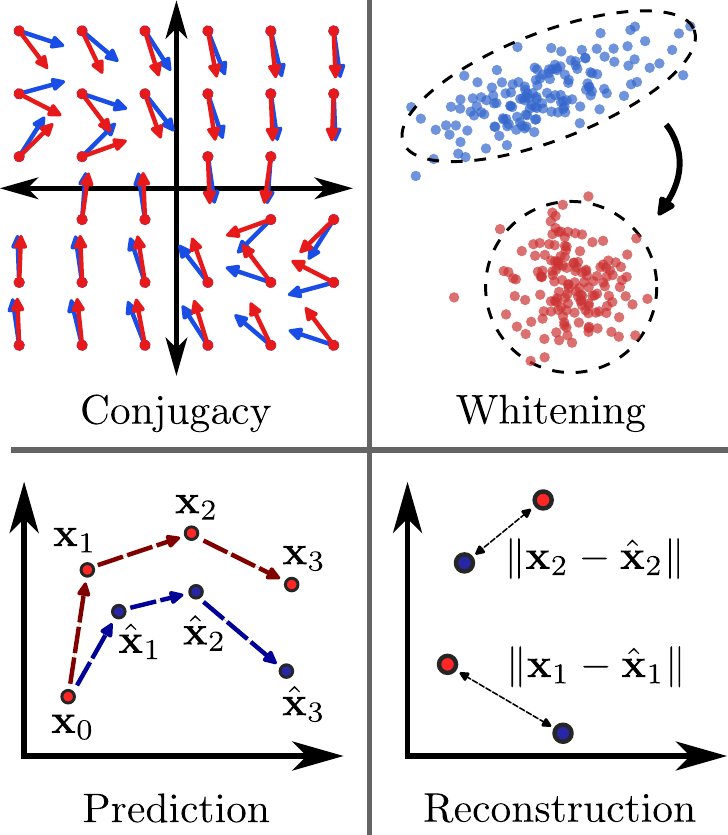}
    \caption{A visual guide of the losses.}
    \label{fig:losses_visual}
    \vspace{-10pt}
\end{wrapfigure}

Using the insights from Section \ref{sec:insights}, we seek parameters \(\boldsymbol\phi\), \(\boldsymbol\psi\), and \(\boldsymbol\rho\) such that the encoder \(\bE_{\boldsymbol\phi}\), decoder \(\bD_{\boldsymbol\psi}\), and latent dynamics model \(\bg_{\boldsymbol\rho}\) approximate the FOM Poincar\'e map:
\begin{align} \label{eq:true_loss_func}
    \inf_{\boldsymbol\phi,\boldsymbol\psi,\boldsymbol\rho}
    \sum_{\bx_k \in \mathfrak{D}}
    \left\|
    \bff(\bx_k) - (\bD_{\boldsymbol\psi}\circ \bg_{\boldsymbol\rho}\circ \bE_{\boldsymbol\phi})(\bx_k)
    \right\|^2.
\end{align}
The encoder, decoder, and latent dynamics networks are trained jointly using a weighted sum of losses, each enforcing a different aspect of reconstruction accuracy and dynamical consistency. Inspired by \cite{lutkus2025latent}, we organize these losses into reconstruction, conjugacy, prediction, whitening, and regularization terms. The loss functions are summarized in Table~\ref{tab:loss-summary} and illustrated in Figure~\ref{fig:losses_visual}.

The reconstruction losses $L_{\mathrm{rec}, x}$ and $L_{\mathrm{rec}, z}$ ensure that the encoder-decoder pair $(\mathbf{E}_{\boldsymbol{\phi}}, \mathbf{D}_{\boldsymbol{\psi}})$ accurately reconstructs states in both the full-order and latent spaces, respectively. The forward and backward conjugacy losses, $L_{\mathrm{fwd}}$ and $L_{\mathrm{bck}}$, encourage the learned latent dynamics to align with the FOM Poincar\'e dynamics. The prediction loss $L_{\mathrm{pred}}$ improves multi-step rollout accuracy, encouraging the learned ROM to capture long-horizon trajectory evolution rather than myopic one-step behavior. The whitening loss $L_{\mathrm{iso}}$ helps prevent latent space collapse by encouraging the latent covariance to remain close to identity. Finally, the weight regularization loss $L_{\mathrm{reg}}$ penalizes large network weights to reduce overfitting to data. The total loss function is given by:
\begin{equation}
    L = L_{\mathrm{rec},x}
      + L_{\mathrm{rec},z}
      + L_{\mathrm{fwd}}
      + L_{\mathrm{bck}}
      + L_{\mathrm{pred}}
      + L_{\mathrm{iso}}
      + L_{\mathrm{reg}}.
\end{equation}
All loss terms are scaled as shown in Table~\ref{tab:loss-summary}, where all weights are set to 1.0 except for the L2 regularization term, which uses $\lambda_{\mathrm{reg}} = 10^{-6}$.

\begin{table}[t]
\centering
\small
\renewcommand{\arraystretch}{1.15}
\setlength{\tabcolsep}{4pt}
\begin{tabular}{p{0.14\linewidth} p{0.58\linewidth} p{0.17\linewidth}}
\hline
\textbf{Loss} & \textbf{Definition} & \textbf{Purpose} \\
\hline
$L_{\mathrm{rec},x}(\bphi,\bpsi)$
&
$
\frac{\lambda_{\mathrm{rec},x}}{BK}
\sum_{b, k=1}^{B, K}
\left\| \bx^{(b)}_{k} - (\bD_{\bpsi}\!\circ\!\bE_{\bphi})(\bx^{(b)}_{k}) \right\|^2
$
&
Reconstruction
\\[6pt]

$L_{\mathrm{rec},z}(\bphi,\bpsi)$
&
$
\frac{\lambda_{\mathrm{rec},z}}{BK}
\sum_{b, k=1}^{B, K}
\left\| \bE_{\bphi}(\bx^{(b)}_{k}) - (\bE_{\bphi}\!\circ\!\bD_{\bpsi}\!\circ\!\bE_{\bphi})(\bx^{(b)}_{k}) \right\|^2
$
&
Reconstruction
\\[6pt]

$L_{\mathrm{fwd}}(\bphi,\boldsymbol\rho)$
&
$
\frac{\lambda_{\mathrm{fwd}}}{B(K-1)}
\sum_{b, k=1}^{B, K-1}
\left\| \bE_{\bphi}(\bx^{(b)}_{k+1}) - (\bg_{\boldsymbol\rho}\!\circ\!\bE_{\bphi})(\bx^{(b)}_{k}) \right\|^2
$
&
Conjugacy
\\[6pt]

$L_{\mathrm{bck}}(\bphi,\boldsymbol\rho,\bpsi)$
&
$
\frac{\lambda_{\mathrm{bck}}}{B(K-1)}
\sum_{b, k=1}^{B, K-1}
\left\| \bx^{(b)}_{k+1} - (\bD_{\bpsi}\!\circ\!\bg_{\boldsymbol\rho}\!\circ\!\bE_{\bphi})(\bx^{(b)}_{k}) \right\|^2
$
&
Conjugacy
\\[6pt]

$L_{\mathrm{pred}}(\bphi,\boldsymbol\rho,\bpsi)$
&
$
\frac{\lambda_{\mathrm{pred}}}{B(K-1)}
\sum_{b=1, k=2}^{B, K}
\left\| \bx^{(b)}_{k} - \big(\bD_{\bpsi}\!\circ\!\bg_{\boldsymbol\rho}^{(k-1)}\!\circ\!\bE_{\bphi}\big)(\bx^{(b)}_{1}) \right\|^2
$
&
Prediction
\\[6pt]

$L_{\mathrm{iso}}(\bphi)$
&
$
\lambda_{\mathrm{iso}}
\left\|
\bI
-
\frac{1}{BK}
\sum_{b, k=1}^{B, K}
\big(\bE_{\bphi}(\bx^{(b)}_{k})-\boldsymbol{\mu}\big)
\big(\bE_{\bphi}(\bx^{(b)}_{k})-\boldsymbol{\mu}\big)^{\!\top}
\right\|_{F}^{2}
$
&
Whitening
\\[6pt]

$L_{\mathrm{reg}}(\bphi,\boldsymbol\rho,\bpsi)$
&
$
\lambda_{\mathrm{reg}}
\sum_{w_i \in \mathbf{W}(\bphi,\bpsi,\boldsymbol\rho)}
\|w_i\|^{2}
$
&
L2 Regularization
\\
\hline
\end{tabular}
\caption{Summary of the loss terms. $\mathbf{W}$ denotes trainable weights (excluding biases), and $\bg_{\boldsymbol\rho}^{(k-1)}$ denotes repeated application of $\bg_{\boldsymbol\rho}$. In addition, $B$ denotes the mini-batch size and $K$ denotes the trajectory length.}
\label{tab:loss-summary}
\vspace{-15pt}
\end{table}
%
%
\paragraph{Training:}
We train autoencoders using the Adam optimizer~\cite{kingma2014adam} with shuffled mini-batches of trajectory data. The model is implemented in JAX and Flax \cite{jax2018github, flax2020github}. Before training, each state feature is normalized independently using training-set statistics aggregated over all trajectories and time steps. The dataset of trajectories is partitioned into disjoint training, validation, and test sets. More details are provided in Appendix~\ref{sec:autoencoder_hyperparams}.
%
\paragraph{Latent Region of Attraction:}
To approximate the region of attraction (ROA) for the FOM impact-to-impact dynamics, we linearize the latent Poincar\'e map around the latent equilibrium point\footnote{Linearization of $\bg_{\boldsymbol\rho}(\bz)$ is obtained via automatic differentiation using \texttt{jax.jacfwd}, giving $\bQ = \frac{\partial \bg_{\boldsymbol\rho}}{\partial \bz}\vert_{\bz^*}$.} $\bz^*$, yielding the discrete-time linear system $\bz_{k+1} = \bg_{\boldsymbol\rho}(\bz_k) \approx \bQ (\bz_k - \bz^*) + \bg_{\boldsymbol\rho}(\bz^*)$. Without loss of generality, we shift coordinates so that $\bz^*=\bzero$. We then construct a discrete-time Lyapunov function $ V_{\bz}(\bz) = \bz^\top \bP \bz$ where $\bP \succ \bzero$ satisfies the discrete-time Lyapunov equation $\bQ^\top \bP \bQ -\bP + \bI = \bzero$.
Next, we construct an estimated ROA around the latent equilibrium and map it through the decoder to obtain an approximation of the full-order ROA. Consider the following set:
\begin{equation}
    \mc{D} = \{\bz \in \mathbb{R}^{n_z} : \Delta V_{\bz}(\bz) \leq 0\} = \Bigl\{ \bz \in \mathbb{R}^{n_z} :
        - \bz^{\top} \bz
        + \bg_{\boldsymbol\rho}(\bz)\T\bP\bg_{\boldsymbol\rho}(\bz) - \bz\T\bQ\T\bP\bQ\bz        \le 0 \Bigr\}.
\end{equation}
This is the set of points for which the latent candidate Lyapunov function $V_{\bz}$ does not increase under the nonlinear latent dynamics.
\noindent To obtain the largest Lyapunov sublevel set contained in $\mc{D}$, we solve:
\begin{equation}\label{eq:roa-nlp}
    c^\ast = \inf_{\bz \in \partial \mc{D}} \bz^\top \bP \bz
\end{equation}
using a Monte Carlo sampling strategy\footnote{Direct optimization of $c^*$ is challenging as $\partial \mathcal D$ is non-convex and implicitly determined by a neural network. We thus use a simple but effective sampling-based sweep to obtain a conservative ROA estimate.}. The sublevel set $\boldsymbol{\Omega}_{c^*} \coloneq \{\bz : \bz^{\top} \bP \bz \leq c^*\}$ is taken as the latent ROA estimate. 
We then sample points in $\boldsymbol{\Omega}_{c^\ast}$, decode them through $\mathbf{D}_{\boldsymbol\psi}$ to obtain full-order states, and simulate these initial conditions under the FOM dynamics to assess if they remain stable.
%


\section{Results}\label{sec:results}
\begin{wrapfigure}{r}{0.45\linewidth}
    \centering
    \vspace{-25pt}
    \includegraphics[width=\linewidth]{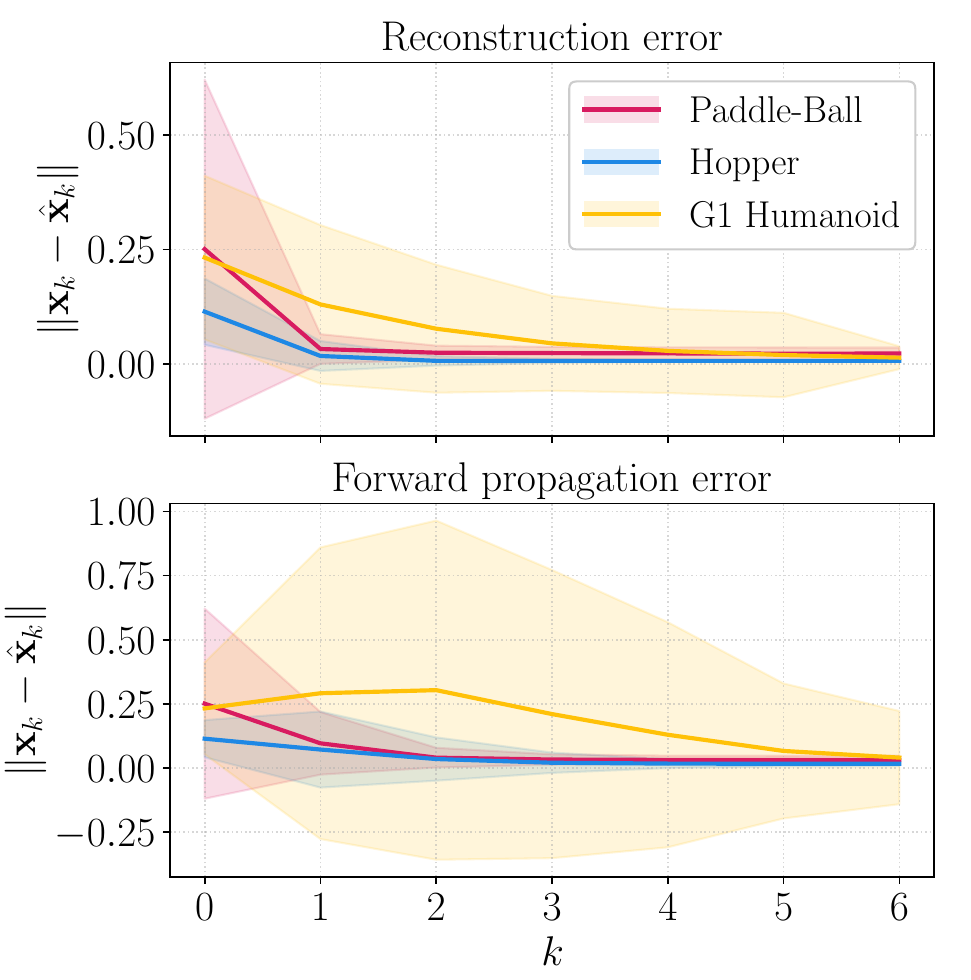}
    \vspace{-18pt}
    \caption{Per-step normalized error of the learned ROM over 3,000 test trajectories for each system. (Top) shows single-step reconstruction error, and (bottom) shows multi-step latent forward-propagation error. Solid curves denote the mean and shaded bands denote $\pm \sigma$.}
    \label{fig:traj_errors}
\end{wrapfigure}

To validate the capabilities of our autoencoder, we test both its prediction and stability estimation capabilities. To test the encoding and decoding accuracy of the autoencoder, we gather normalized Poincar\'{e} data of $N$ full-order rollouts, each of length $K$ for each of our systems. We send this data through the encoder followed by the decoder and calculate 
\begin{align}
    \ell_{\mathrm{rec}} =\lVert \bx_k-\mathbf{D}_{\boldsymbol\psi} \circ \bE_{\bphi}(\bx_k) \rVert.
\end{align} 
For each state in the trajectory, we take the mean and standard deviation of $\ell_{\mathrm{rec}}$ at each Poincar\'{e} step $k$. Second, to test the accuracy of the latent dynamics, we take the initial states $\bx_0$ for each of our rollouts, encode the state, perform $K$ steps of the latent dynamics, then decode the latent trajectory, calculating the reconstruction loss as 
\begin{align}
    \ell_{\mathrm{dyn}}= \lVert \bx_k - \mathbf{D}_{\boldsymbol\psi}\circ\bg_{\boldsymbol\rho}^{(k)}\circ\bE_{\bphi}(\bx_0)\rVert.
\end{align}
We also take the mean and standard deviation over the Poincar\'{e} steps. We observe good performance of the reconstruction and latent dynamics, improving over the rollout as the system stabilizes. We notice that error propagation is minimal despite the length of the rollout, which we attribute to our loss choice penalizing discrepancy in long-horizon prediction, combined with the stability of the systems. 


To validate our ROA estimate, we perform the procedure described in Sec. \ref{sec:methods}: we generate a Lyapunov-based ROA estimate by solving \eqref{eq:roa-nlp}, sampling 2,000 points in $\boldsymbol\Omega_{c^*}$, and mapping them back to the full-order state space. For comparison, we compute an approximation of the ROA using a naive sampling-based method. Here, we grid the latent space with 4 points per dimension in $n_z$ dimensions, gridding a hypercube of side length equal to the major axis length of the Lyapunov region, decoding these points using the decoder. Next, we roll out all of these samples in parallel simulation, recording those that are stable, characterized by $p_z > p_z^{\mathrm{thresh}}$ not falling below a prescribed threshold over the course of the 150-step trajectory. We observe that the initial conditions in the Lyapunov-based latent ROA consistently have a near-perfect $99.9 \pm 0.1\%$ stability rate, while points from the naive hypercube sampling method have a consistently lower stability rate, for instance $75.6\ \pm 10.1\% $ on the G1. In Figures \ref{fig:FullOrderTransferG1} and \ref{fig:FullOrderTransferHopper}, we plot the stable points from each of these categories. We find that, while the Lyapunov estimate is more conservative, it is far more accurate than the naive sampling-based estimate of the ROA. In Figure \ref{fig:hardware_walk}, we demonstrate the prediction performance of the ROM relative to real-world Unitree G1 experiments.
\begin{figure}[h]
\centering
\includegraphics[width=0.7\linewidth]{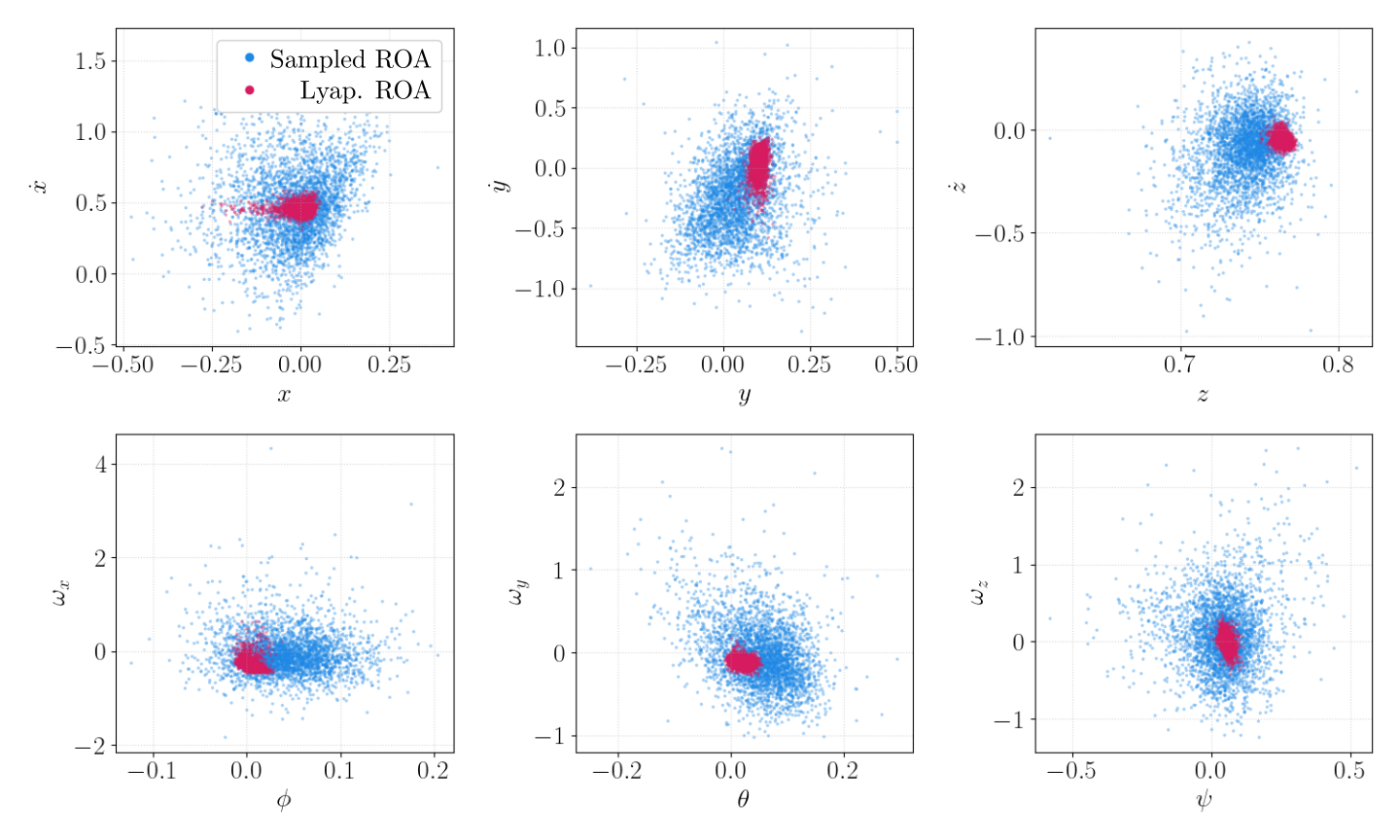}
\vspace{-10pt}
\caption{Decoded visualization of the boundary of $\bOmega_{\bz}$ in full-order coordinates of the G1. Plotted states correspond to the center-of-mass position relative to the stance foot.}
\label{fig:FullOrderTransferG1}
\end{figure}

\begin{figure}[h]
\centering
\vspace{-10pt}
\includegraphics[width=0.85\linewidth]{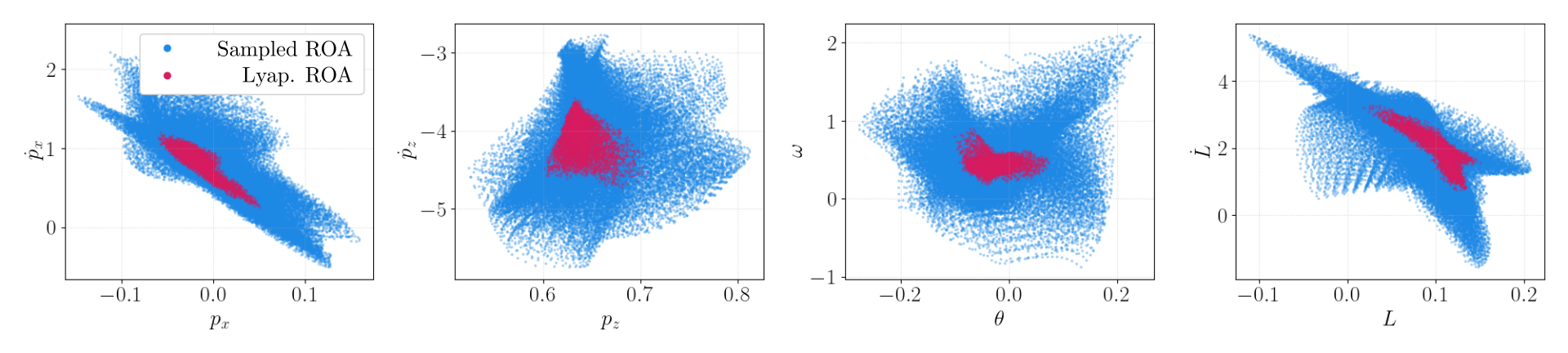}
\vspace{-10pt}

\caption{Decoded visualization of the boundary of $\bOmega_{\bz}$ in full-order coordinates of the hopper. Plotted states correspond to the center-of-mass position relative to the hopper foot.}
\label{fig:FullOrderTransferHopper}
\end{figure}

\begin{figure}[h]
\centering
\vspace{-10pt}
\includegraphics[width=0.76\linewidth]{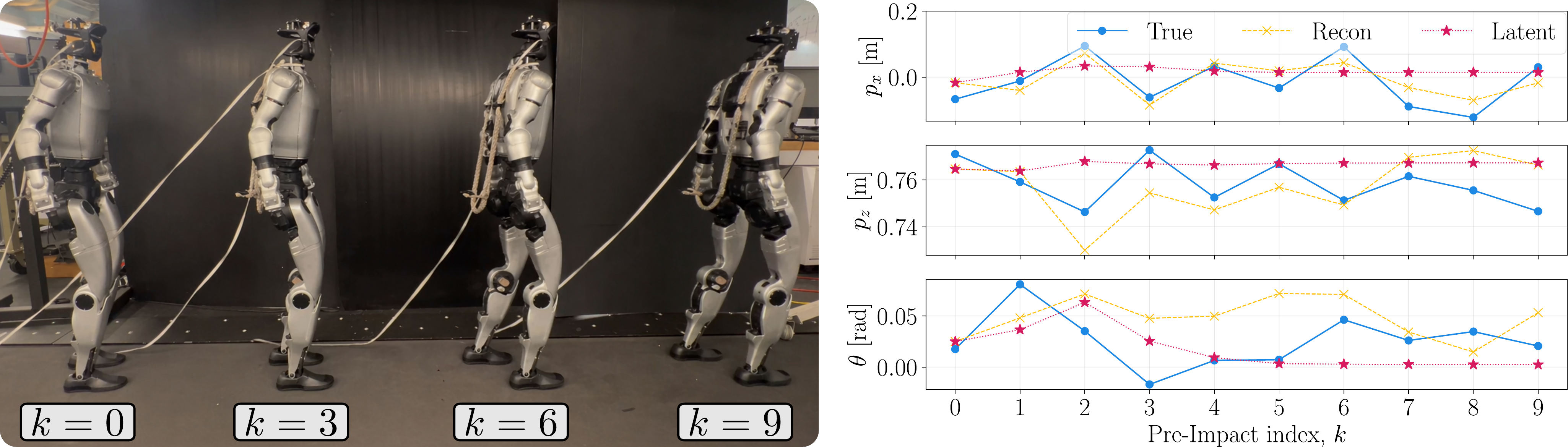}
\vspace{-5pt}
\caption{A 9-step rollout with true hardware, encoded-decoded, and latent ROM trajectories.}
\vspace{-10pt}
\label{fig:hardware_walk}
\end{figure}

\paragraph{Conclusion:}
In this work, we proposed a method of generating reduced-order models using autoencoders to represent and study the stability of hybrid systems. Our technique enables accurate reconstruction of the original dynamical system from a lower-dimensional latent space, as well as estimation of the region of attraction of the original system. This method was validated on three systems of increasing complexity: a paddle-ball, a hopper, and a Unitree G1 humanoid robot.

\clearpage

\clearpage

\acks{We would like to thank Paul Lutkus,  Kaiyuan Wang, and Professor Steven Tu for their insights and guidance. This research is supported by the Technology Innovation Institute.}


\bibliography{References/References}

\clearpage
\appendix
\section{Full-Order Model  Hybrid Dynamics}\label{sec:fom_hybrid_dynamics}
Legged robots that make and break contact with the environment can be modeled as nonlinear dynamical systems with impulse effects, which in this work is referred to as the \textit{full-order model}. The generalized configuration for these systems is $\bq \in \mc{Q} \subseteq \R^{n}$ with state $\bx = (\bq, \dot{\bq}) \in \mathsf{T} \mc{Q} \subseteq \R^{2n} = \R^{n_x}$. Using Euler-Lagrange equations, we define the continuous dynamics:
\begin{align} 
    \bM(\bq) \ddot{\bq} + \bH(\bq, \dot{\bq}) &= \bB \bu + \bJ^{\top}_c(\bq) \boldsymbol{\lambda}, \label{eq:euler_lagrange_dyn} \\
    \bJ(\bq) \ddot{\bq} + \dot{\bJ}(\bq, \dot{\bq}) \dot{\bq} &= \bzero \label{eq:holo_cons}
\end{align}
where $\bM : \mc{Q} \rightarrow \R^{n\times n}$ is the positive definite mass-inertia matrix, $\bH : \R^{n_x} \rightarrow \R^n$ contains centrifugal, Coriolis, and gravitational terms, $\bB \in \R^{n \times n_u}$ is the actuation matrix, and $\bu \in \mc{U} \subseteq \R^{n_u}$ is the control input. Additionally, $\bJ : \mc{Q} \rightarrow \R^{n_c \times n}$ is the contact Jacobian of a \textit{holonomic constraint} describing contact of the robot with the environment with corresponding constraint wrench $\boldsymbol{\lambda} \in \R^{n_c}$. Furthermore, \eqref{eq:euler_lagrange_dyn} and \eqref{eq:holo_cons} can be combined to write the state-space dynamics representation:
\begin{equation}\label{eq:fom_continuous_dyn}
    \dot{\bx} = 
    \mathbf{f}_c(\bx, \bu) \coloneq
    \begin{bmatrix}
        \dot \bq \\
        -\mathbf{M}^{-1}(\bq) \tilde{\mathbf{H}}(\bq, \dot \bq)
    \end{bmatrix}
    +
    \begin{bmatrix}
        \mathbf{0} \\
        \mathbf{M}^{-1}(\bq) \tilde{\mathbf{B}}(\bq)
    \end{bmatrix}
    \bu,
\end{equation}
where the dynamics $\mathbf{f}_{c}: \R^{n_x} \times \R^{n_u} \rightarrow \R^{n_x}$ are continuously differentiable. 

Discrete jumps in hybrid systems occur when a foot or hand impacts with the environment. A \textit{switching surface} is designed to capture such events. Given a continuously differentiable function $s : \R^{n_x} \rightarrow \R$, let $\mc{S} \subset \R^{n_x}$ denote the switching surface as:
\begin{align}
    \mc{S} = \{\bx \in \R^{n_x} \; | \; s(\bx) = 0, \dot{s}(\bx) < 0\}.
\end{align}
The dynamics undergo discrete transitions which are represented by:
\begin{equation}\label{eq:fom_reset_map}
\bx^+ = \bm{\Delta}(\bx^-)
\end{equation}
where  $\bm{\Delta}\,:\, \mc{S} \rightarrow \R^{n_x}$ is the continuously differentiable \textit{reset map}, $\bx^{-}\in\mc{S}$ denotes the \textit{pre-impact state} just before impact and $\bx^{+}\in\R^{n_x}\setminus\mc{S}$ denotes the \textit{post-impact state} just after impact.
Combining the continuous dynamics \eqref{eq:fom_continuous_dyn} and the reset map \eqref{eq:fom_reset_map}, the full-order dynamics can be written as a hybrid control system:
\begin{equation} \label{eq:fom_HC}
        \mathcal{HC}
        =
        \left\{
        \begin{array}{ll}
              \dot {\bx} = \mathbf{f}_c(\bx, \bu) & \text{if} \quad \bx \in \R^{n_x} \backslash \mc{S}\\
              \bx^{+} = \bm{\Delta}(\bx^{-})    & \text{if} \quad \bx^{-} \in \mc{S}.
        \end{array} 
        \right. 
\end{equation}
Furthermore, given a locally Lipschitz feedback controller $\bk\,:\,\R^{n_x}\rightarrow\R^{n_u}$ we obtain the closed-loop hybrid system:
\begin{equation} \label{eq:fom_H}
        \mathcal{H}
        =
        \left\{
        \begin{array}{ll}
              \dot {\bx} = \mathbf{f}_c(\bx, \bk(\bx)) & \text{if} \quad \bx \in \R^{n_x} \backslash \mc{S}\\
              \bx^{+} = \bm{\Delta}(\bx^{-})    & \text{if} \quad \bx^{-} \in \mc{S}.
        \end{array} 
        \right. 
\end{equation}
Let the flow of these dynamics be denoted by $\bm{\varphi}_{t}(\bx)$, which returns the state reached at time $t$ under the continuous dynamics \eqref{eq:fom_continuous_dyn} when starting from state $\bx\in\mc{X}$. To further characterize the discrete evolution of \eqref{eq:fom_H}, define the \textit{time-to-impact function} $T_{I}\,:\,\tilde{\mc{S}}\rightarrow\R_{>0}$ as:
\begin{equation}\label{eq:fom-time-to-impact}
    T_I(\bx) \coloneqq \inf\{t \geq 0\; | \; \boldsymbol\varphi_t( \bm\Delta(\bx)) \in \mc{S}\},
\end{equation}
where $\tilde{\mc{S}} \coloneqq \{\bx\in\mc{S}\,|\,T_{I}(\bx)\in(0,\infty) \}$, which captures the time passed between discrete transitions. The above formulation allows for casting the hybrid system \eqref{eq:fom_H} as a discrete-time system via the \textit{Poincar\'e map}:
\begin{equation}\label{eq:fom_poincare}
    \bx_{k+1} = \mathbf{f}(\bx_{k}),\, \quad \mathbf{f}(\bx) \coloneqq \bm{\varphi}_{T_{I}(\bx)}\big(\bm{\Delta}(\bx)\big),\, \quad \mathbf{f}\,:\,\tilde{\mc{S}}\rightarrow\mc{S}.
\end{equation}
In the literature, for legged robots, the Poincar\'e map $\bx_{k+1} = \mathbf{f}(\bx_k)$ is often referred to as the \textit{step-to-step dynamics} of the FOM. With an abuse of notation, denote the Poincar\'e map as $\mathbf{f}\,:\,\mc{S} \rightarrow \mc{S}$ with the understanding that $\mathbf{f}$ is a partial function (see \cite{ames2014rapidly,westervelt2003hybrid}). The Poincar\'e map provides a powerful tool for studying the stability of orbits of \eqref{eq:fom_H}. The flow of \eqref{eq:fom_H} is said to be periodic with period $T\geq0$ if there exists a state $\bx^*\in\mc{S}$ such that $\flow_T(\bm{\Delta}(\bx^*))=\bx^*$. This induces a \textit{periodic orbit}:
\begin{equation}\label{eq:hybrid-periodic-orbit}
    \mc{O} \coloneqq \{ \flow_t(\bm{\Delta}(\bx^*))\,|\,t\in[0,T],\,T=T_{I}(\bx^*)\}.
\end{equation}
Importantly, stability of $\mc{O}$ is equivalent to that of $\mathbf{f}$ as shown in the following theorem.
\begin{theorem}[\cite{westervelt2018feedback}]
    Let the closed loop dynamics $\mathbf{f}$ and reset map $\bm{\Delta}$ be continuously differentiable and $\overline{\bm{\Delta}(\mc{S})} \cap \mc{S} = \emptyset$. Then, the following statements hold:
    \begin{itemize}
        \item $\bx^{*} \in \mc{S}$ is a stable equilibrium point for $\bx_{k+1} = \mathbf{f}(\bx_{k})$ $\iff$ the orbit $\mc{O}(\bm{\Delta}(\bx^{*}))$ is stable.
        \item $\bx^{*} \in \mc{S}$ is an asymptotically stable equilibrium point for $\bx_{k+1} = \mathbf{f}(\bx_{k})$ $\iff$ the orbit $\mc{O}(\bm{\Delta}(\bx^{*}))$ is asymptotically stable.
        \item $\bx^{*} \in \mc{S}$ is an exponentially stable equilibrium point for $\bx_{k+1} = \mathbf{f}(\bx_{k})$ $\iff$ the orbit $\mc{O}(\bm{\Delta}(\bx^{*}))$ is exponentially stable.
    \end{itemize}
    \label{thm:poincare_stability}
\end{theorem}
Thus, establishing stability of the Poincar\'e map \eqref{eq:fom_poincare} certifies stability of the full hybrid system \eqref{eq:fom_H}.
%
%
\section{Proofs}\label{appendix:proofs}
In this section, we provide proofs of the theoretical results in the paper. For convenience, we restate all results below.

\subsection{Proof of Lemma \ref{lem:one-step}}
\begin{lemma*}
    Let $\bA \in \R^{n_{x}\times n_x}$ and $n_z \in (0, \rank(\bA)]_{\mathbb Z}$. Consider:
    \begin{align}
        \inf_{\bE \in \R^{n_z \times n_x}, \bD \in \R^{n_x \times n_z}, \bQ \in \R^{n_z \times n_z}} \sup_{\bx_0 \in \R^{n_x}, \norm{\bx_0} = 1}\norm{\bA \bx_0 - \bD \bQ \bE \bx_0}^2.
    \end{align}
     If $\bA = \bU \bSigma \bV^\top$ is an SVD of $\bA$, then $\bD = \bU_{n_z}$, $\bQ = \bSigma_{n_z}$, $\bE = \bV_{n_z}^\top$, where $\bU_{n_z}, \bV_{n_z}$ are the first $n_z$ columns of $\bU, \bV$ and $\bSigma_{n_z} = \diag(\sigma_1, ..., \sigma_{n_z})$, are optimal solutions.
\end{lemma*}
\begin{proof}
    The maximization, $\sup_{\bx_0 \in \R^{n_x}, \norm{\bx_0} = 1} \norm{\bA \bx_0 - \bD \bQ \bE \bx_0}^2$ (for $\norm{\cdot}$ the $\ell_2$-norm) is simply the square of the induced matrix $2$-norm, $\norm{\bA - \bD\bQ\bE}_2^2$. Hence, the problem reduces to 
    \begin{align}
        \inf_{\bE, \bD, \bQ} \norm{\bA - \bD \bQ \bE}_2^2.
    \end{align}
    Since $\rank(\bD \bQ \bE) \leq n_z$ for every pair of $\bD, \bQ, \bE$, it follows that
    \begin{align}
         \inf_{\bE, \bD, \bQ} \norm{\bA - \bD \bQ \bE}_2^2 &\geq \inf_{\bM \in \R^{n_x \times n_x}} \norm{\bA - \bM}_2^2 \text{ s.t. } \rank(\bM) \leq n_z. 
    \end{align}
    The lower bound is a well-known rank-constrained minimization \cite[p. 134]{calafiore2014optimization}. For an SVD $\bA = \bU \bSigma \bV^\top$, the Eckart-Young theorem tells us that an optimal solution $\bM$ is the $n_z$-truncation of the SVD, $\bM = \bU_{n_z} \bSigma_{n_z} \bV_{n_z}^\top$. Here, $\bU_{n_z}, \bV_{n_z}$ are the first $n_z$ columns of $\bU, \bV$, respectively, and $\bSigma_{n_z} = \diag(\sigma_1, ..., \sigma_{n_z})$. Since $\bD = \bU_{n_z}, \bQ = \bSigma_{n_z}, \bE = \bV_{n_z}^\top$ are feasible for the original problem, the lower bound above implies their optimality.
\end{proof}

\subsection{Proof of Lemma \ref{lem:diag-model-red}}
\begin{lemma*}
     Suppose $\bA = \diag(\lambda_1, ..., \lambda_{n_x})$, with $|\lambda_1| \geq ... \geq |\lambda_{n_x}|$. Let $n_z \in (0, \rank(\bA)]_{\mathbb Z}$. If the unstable and center subspaces of $\bA$ together have dimension $\leq n_z$, then Problem \eqref{eq:model-red-prob} is feasible and an optimal solution is given by the $n_z$-truncation:
    \begin{align}
        \hat \bD = \begin{bmatrix}
            I_{n_z}\\
            0_{(n_x-n_z) \times n_z}
        \end{bmatrix}, \; \hat \bQ = \diag(\lambda_1, ..., \lambda_{n_z}), \; \hat \bE = \begin{bmatrix}
            I_{n_z} & 0_{n_z \times (n_x-n_z)}
        \end{bmatrix}.
    \end{align}
\end{lemma*}
\begin{proof}
    We begin by considering the simpler one-step problem $\inf_{\bE, \bD, \bQ} \norm{\bA^k - \bD \bQ^k \bE}_F^2$. Since $\rank(\bD \bQ^k \bE) \leq n_z$, it must be that for each fixed $k \geq 0$,
    \begin{align}
        \inf_{\bE, \bD, \bQ} \norm{\bA^k - \bD \bQ^k \bE}_F^2 \geq \inf_{\bM} \norm{\bA^k - \bM}_F^2, \text{ s.t. } \rank(\bM) \leq n_z.
    \end{align}
    Now, fix $k \geq 0$. Since $\bA$ is diagonal, $\bA^k = \diag(\lambda_1^k, ..., \lambda_{n_x}^k)$. As such, any off-diagonal entry $\bM_{ij}$ raises the Frobenius norm error. This implies that we can further reduce the problem to:
    \begin{align}
        \begin{pmatrix}
            \inf_{\bM} \norm{\bA^k - \bM}_F^2\\
            \text{ s.t. } \rank(\bM) \leq n_z
        \end{pmatrix}
        =&
        \begin{pmatrix}
            \inf_{\bM} \norm{\bA^k - \bM}_F^2\\
            \text{ s.t. } \rank(\bM) \leq n_z, \; \bM \text{ diag.}
        \end{pmatrix}
    \end{align}
    Since $\bM$ is constrained to have rank $\leq n_z$, the problem amounts to choosing at most $n_z$ nonzero diagonal entries of $\bM$. Since the entries of $\bA$ are ordered by decreasing magnitude an optimal $\bM$ is $\bM = \diag(\lambda_1^k, ..., \lambda_{n_z}^k, 0_{n_x - n_z})$. Since $\bM = \hat \bD \hat \bQ^k \hat \bE$, for $\hat \bD, \hat \bQ, \hat \bE$ as in the statement of the lemma,
    \begin{align}\label{eq:lower-bound}
        \inf_{\bE, \bD, \bQ} \norm{\bA^k - \bD \bQ^k \bE}_F^2 &\geq \norm{\bA^k - \hat \bD \hat \bQ^k \hat \bE}_F^2, \; \forall k \geq 0.
    \end{align}
    If the unstable and center subspaces of $\bA$ together have dimension $\leq n_z$, for each $j > n_z$, $|\lambda_j| < 1$. Hence, one has
    \begin{align}
        \sum_{k = 0}^\infty \norm{\bA^k - \hat \bD \hat \bQ^k \hat \bE}_F^2 &= \sum_{k = 0}^\infty \sum_{j = n_z + 1}^{n_x} |\lambda_j^k|^2 =  \sum_{j = n_z + 1}^{n_x} \sum_{k = 0}^\infty |\lambda_j^k|^2 < \infty,
    \end{align}
    which confirms feasibility of the proposed solution. Applying \eqref{eq:lower-bound} for each $k$ yields
    \begin{align}
        \inf_{\bE, \bD, \bQ} \sum_{k = 0}^\infty \norm{\bA^k - \bD \bQ^k \bE}_F^2 \geq \sum_{k = 0}^\infty \norm{\bA^k - \hat \bD\hat \bQ^k \hat \bE}_F^2,
    \end{align}
    which establishes the optimality of the proposed solution.
\end{proof}

\subsection{Proof of Lemma \ref{lem:inv-manif}}
\begin{lemma*}
    Let $\bff: \R^{n_x} \to \R^{n_x}$ be a map defining a discrete-time system with smooth invariant manifold $\M \subseteq \R^{n_x}$. For $\dim(\M) = k$:
    \begin{enumerate}
        \item \underline{Perfect encoder-decoder pair on $\M$}: there exists a smooth encoder $\bE: U \subseteq \R^{n_x} \to \R^{n_z}$ and a smooth decoder $\bD: W \subseteq \R^{n_z} \to \R^{n_x}$ satisfying $\bD(\bE(\bx)) = \bx$ for all $\bx \in \M$, where $n_z \leq 2k$ and $U, W$ are open neighborhoods of $\M$ and $\bE(\M)$ respectively, 
        \item \underline{Latent system}: there exists a map $\bg: \R^{n_z} \to \R^{n_z}$ satisfying $\bff(\bx) = \bD \circ \bg \circ \bE(\bx) \; \forall \bx \in \M$.
    \end{enumerate}
\end{lemma*}
\begin{proof}
    Whitney's embedding theorem guarantees the existence of a smooth embedding $\overline \bE: \M \to \R^{n_z}$, where $n_z \leq 2k$. Since $\overline \bE$ is an embedding, we may restrict its codomain to obtain a diffeomorphism $\tilde \bE: \M \to \overline \bE(\M)$ . Since $\tilde \bE$ is a diffeomorphism, it has a smooth inverse $\tilde \bD: \overline \bE(\M) \to \M$, satisfying $\tilde \bD(\tilde \bE(\bx)) = \tilde \bD(\overline \bE(\bx)) = \bx$ for all $\bx \in \M$.

    Now, we extend the (co)domains of $\overline \bE$ and $\tilde \bD$. By \cite[Lem. 5.34]{lee2013smooth}, there exists a neighborhood $U \subseteq \R^{n_x}$ of $\M$ and a smooth function $\bE: U \to \R^{n_z}$ for which $\bE|_{\M} = \overline \bE$. Since $\M$ is embedded in $\R^{n_x}$, the inclusion map $\iota_\M: \M \hookrightarrow \R^{n_x}$ taking $\bx \in \M$ to $\bx$ as an element of $\R^{n_x}$ is smooth; therefore the composition $\iota_{\M} \circ \tilde \bD : \overline \bE(\M) \to \R^{n_x}$ is smooth. Applying \cite[Lem. 5.34]{lee2013smooth} to $\iota_{\M} \circ \tilde \bD$, we conclude the existence of a neighborhood $W \subseteq \R^{n_z}$ of $\overline \bE(\M)$ and a map $\bD: W \to \R^{n_x}$ for which $\bD|_{\overline \bE(\M)} = \iota_{\M} \circ \tilde \bD$. By the extension property, it follows that $\bD(\bE(\bx)) = \bx$ for all $\bx \in \M$, which establishes (1).

    Now, we construct the dynamical system $\bz_{k+1} = \bg(\bz_k)$, satisfying $\bff(\bx) = \bD(\bg(\bE(\bx)))$ for all $\bx \in \M$. Define $\tilde \bg: \bE(\M) \to \R^{n_z}$ by $ \bz \mapsto \bE \circ \bff \circ \bD(\bz)$. If $\bz \in \bE(\M)$, then $\bD(\bz) \in \M$ by part (1), which implies that $\bE \circ \bff \circ \bD(\bz) \in \bE(\M)$. So, $\tilde \bg(\bz) = \bE \circ \bff \circ \bD(\bz)$ and $\bE \circ \bff \circ \bD(\bz) \in \bE(\M)$ implies that $\bD \circ \tilde \bg(\bz) = \bff \circ \bD(\bz)$ for all $\bz \in \bE(\M)$. We conclude that $\bD \circ \tilde \bg \circ \bE(\bx) = \bff(\bx)$ for all $\bx \in \M$. Extending $\tilde \bg$ to $\R^{n_z}$ (we don't ask for regularity in $\R^{n_z} \setminus \bE(\M)$), the result follows.
\end{proof}

\subsection{Proof of Theorem \ref{thm:transf-stb}}
\begin{theorem*}
    Consider the setting of Lemma \ref{lem:inv-manif} and its proof. Let $\tilde \bg : \bE(\M) \to \bE(\M)$ be the restriction of $\bg$ to $\bE(\M)$ and $\tilde \bff: \M \to \M$ the restriction of $\bff$ to $\M$. Suppose $\bz^* \in \bE(\M)$ is a locally asymptotically stable fixed point of $\tilde \bg$. Then, $\bx^* = \bD(\bz^*)$ is a locally asymptotically stable fixed point of $\tilde \bff$. Further, if $\bz$ belongs to the region of attraction of $\bz^*$ for $\tilde \bg$, then $\bD(\bz)$ belongs to the region of attraction of $\bx^*$ for $\tilde \bff$.
\end{theorem*}
\begin{proof}
    In the proof of Lemma \ref{lem:inv-manif}, we demonstrated that the map $\tilde \bE: \M \to \bE(\M)$ is a diffeomorphism with inverse $\tilde \bD: \bE(\M) \to \M$. Further, we showed that $\tilde \bff = \tilde \bD \circ \tilde \bg \circ \tilde \bE$. First, we note that this conjugacy implies $\tilde \bff^{(j)} = \tilde \bD \circ \tilde \bg^{(j)} \circ \tilde \bE$ for each $j \geq 0$, where $(j)$ represents $j$ compositions (with $j = 0$ being the identity map). 
    
    This composition equality implies that $\bx^* = \bD(\bz^*)$ is a fixed point of $\tilde \bff$. As local asymptotic stability of fixed points is preserved by diffeomorphism, we conclude local asymptotic stability. Now, we consider the domain of attraction. If $\bz$ belongs to the region of attraction of $\bz^*$, then $\lim_{j \to \infty}\tilde \bg^{(j)}(\bz) = \bz^*$. From the composition rule above, $\tilde \bff^{(j)}(\tilde \bD(\bz)) = \tilde \bD \circ \tilde \bg^{(j)} \circ \tilde \bE \circ \tilde \bD(\bz) = \tilde \bD \circ \tilde \bg^{(j)}(\bz)$ for each $j \geq 0$. Passing to the limit, $\lim_{j \to \infty} \tilde \bff^{(j)}(\tilde \bD(\bz)) = \tilde \bD(\lim_{j \to \infty} \tilde \bg^{(j)}(\bz)) = \tilde \bD(\bz^*) = \bx^*$. We conclude that $\tilde \bD(\bz) = \bD(\bz)$ belongs to the region of attraction of $\bx^*$ for $\tilde \bff$.
\end{proof}
%
%
\section{Implementation Details}
This section contains implementation details spanning the entire proposed pipeline. The full code implementation can be found here: \url{https://github.com/sesteban951/halo-latent-locomotion}.
\subsection{Reinforcement Learning Parameters} \label{sec:rl_policy}
For all three systems, the RL policy serves as the feedback controller used to stabilize the closed-loop dynamics and generate trajectories for subsequent ROM learning. Table~\ref{tab:rl_details} summarizes the controller implementation details. Additionally, because the three systems differ substantially in morphology and task structure, their observation spaces are specified separately in Tables~\ref{tab:obs_paddle_ball}--\ref{tab:obs_g1}.
%
\begin{table}[h]
    \centering
    \begin{tabular}{l c c c}
        \hline
        \textbf{Property} & \textbf{Paddle-Ball} & \textbf{Hopper} & \textbf{G1 Humanoid} \\
        \hline
        Framework & Brax & Brax & mjlab\\
        Algorithm & PPO & PPO & PPO \\
        Observation dim. & 4 & 8 & 80 \\
        Action dim. & 1 & 2 & 23 \\
        Action type & Direct force & Direct torque/force & Position setpoints (PD) \\
        Policy network & $[32] \times 4$ & $[32] \times 4$ & $[512,256,128]$ \\
        Value network & $[256]\times 5$ & $[256]\times 5$ & $[512,256,128]$ \\
        Activation & Swish & Swish & ELU \\
        Control frequency & 50 Hz & 50 Hz & 50 Hz \\
        Number of envs & 2048 & 4096 & 4096 \\
        Training steps & 100M & 60M & $980$M \\
        Episode length & 10 sec. & 12 sec. & 20 sec. \\
        \hline
    \end{tabular}
    \caption{RL controller implementation details.}
    \label{tab:rl_details}
\end{table}
\begin{table}[h]
    \centering
    \begin{tabular}{l c l }
        \hline
        \textbf{Observation} & \textbf{Size} & \textbf{Purpose} \\
        \hline
        Ball position   & 1 & Track ball height for interception \\
        Paddle position & 1 & Track paddle location \\
        Ball velocity   & 1 & Predict incoming ball motion \\
        Paddle velocity & 1 & Measure paddle motion \\
        \hline
        \textbf{Total}  & \textbf{4} & \\
        \hline
    \end{tabular}
    \caption{Paddle-ball RL observations.}
    \label{tab:obs_paddle_ball}
\end{table}
\begin{table}[h]
    \centering
    \begin{tabular}{ l c l}
        \hline
        \textbf{Observation} & \textbf{Size} & \textbf{Purpose} \\
        \hline
        Body height         & 1 & Regulate vertical posture \\
        Body orientation    & 2 & Regulate torso orientation \\
        Leg position        & 1 & Measure leg configuration \\
        Horizontal velocity & 1 & Track forward motion \\
        Vertical velocity   & 1 & Track vertical motion \\
        Angular velocity    & 1 & Measure body rotation rate \\
        Leg velocity        & 1 & Measure leg motion \\
        \hline
        \textbf{Total}      & \textbf{8} & \\
        \hline
    \end{tabular}
    \caption{Hopper RL observations.}
    \label{tab:obs_hopper}
\end{table}
\begin{table}[h]
    \centering
    \begin{tabular}{ l c l }
        \hline
        \textbf{Observation} & \textbf{Size} & \textbf{Purpose} \\
        \hline
        Angular velocity   & 3  & Measure base rotational motion \\
        Projected gravity  & 3  & Regulate body orientation w.r.t. gravity \\
        Velocity command   & 3  & Specify desired locomotion velocity command \\
        Gait phase         & 2  & Encode periodic gait timing \\
        Joint positions    & 23 & Measure robot configuration \\
        Joint velocities   & 23 & Measure joint motion \\
        Last action        & 23 & Provide action history for smoothing \\
        \hline
        \textbf{Total}     & \textbf{80} & \\
        \hline
    \end{tabular}
    \caption{G1 humanoid RL observations.}
    \label{tab:obs_g1}
\end{table}
%
%
\subsection{Autoencoder Parameters} \label{sec:autoencoder_hyperparams}
We train each autoencoder using the Adam optimizer. Prior to training, the Poincaré section data is normalized to zero mean and unit variance per feature, ensuring that all state components contribute equally to the loss regardless of their physical units or magnitudes. The remaining hyperparameters, which vary across the three systems to account for differences in state dimension and dataset size, are summarized in Table~\ref{tab:ae_details}.
\begin{table}[h]
    \centering
    \begin{tabular}{l c c c}
        \hline
        \textbf{Property} & \textbf{Paddle-Ball} & \textbf{Hopper} & \textbf{G1 Humanoid} \\
        \hline
        FOM state dim.,\;$n_x$       & 4   & 8    & 59  \\
        ROM state dim.,\;$n_z$          & 2   & 4    & 12  \\
        Encoder layers,\;$\bE_{\boldsymbol\phi}$              & $[64, 32, 16]$ & $[64, 32, 16]$ & $[256, 128, 64]$ \\
        Decoder layers,\;$\bD_{\boldsymbol\psi}$              & $[16, 32, 64]$ & $[16, 32, 64]$ & $[64, 128, 256]$ \\
        Dynamics layers,\;$\bg_{\boldsymbol\rho}$             & $[64]\times3$ & $[64]\times3$ & $[128]\times3$ \\
        Activation                  & Swish & Swish & Swish \\
        Training steps              & 5000  & 15000  & 30000 \\
        Learning rate               & $10^{-3}$ & $10^{-3}$ & $10^{-3}$ \\
        Trajectory length,\;$K$       & 6   & 8    & 8   \\
        Mini-batch size,\;$B$                  & 512 & 1024 & 512 \\
        Training Trajectories  & 40960 & 81920 & 108403 \\ 
        Testing Trajectories  & 8192 & 8192 & 36230 \\ 
        Validation fraction,\;\%         & 5\% & 10\% & 17\% \\
        \hline
    \end{tabular}
    \caption{Autoencoder ROM training details. The G1 state vector is 
    59-dimensional due to the quaternion representation of base orientation 
    ($n_q = 30$, $n_v = 29$), though the underlying state space 
    is 58-dimensional.}
    \label{tab:ae_details}
\end{table}

We normalize each state feature independently using training-set statistics aggregated over all trajectories and time steps, yielding
$\bx~=(\bx-\boldsymbol\mu)/\boldsymbol\sigma$. This standardization reduces variation in scale across state components and improves the conditioning of the learning problem. Without this normalization, state components with larger magnitudes would dominate the loss and bias learning away from smaller-scale but still informative dynamics.

During training, the autoencoder operates on these normalized state trajectories. The same feature-wise normalization is applied consistently across the training, validation, and test splits so that the encoder, decoder, and latent dynamics model all learn on a common normalized representation of the FOM state. For evaluation, we reuse the same training-set statistics to scale residuals on a per-feature basis. In addition, for the plotting metric, we normalize by state dimension when computing the per-step RMS error across state components. This yields a normalized, dimensionless error measure that is more comparable across systems and state dimensions.
%
%

\end{document}